\documentclass[11pt]{article}
\usepackage[letterpaper,margin=1in]{geometry}

\usepackage{booktabs} 
\usepackage[ruled]{algorithm2e} 

\SetAlFnt{\small}
\SetAlCapFnt{\small}
\SetAlCapNameFnt{\small}
\SetAlCapHSkip{0pt}
\IncMargin{-\parindent}
\usepackage{graphicx} 
\usepackage{amsfonts}
\usepackage{commath}
\usepackage{bbm}
\usepackage{amsmath}
\usepackage{enumitem}
\usepackage{amsthm}
\usepackage{bm}
\newcommand{\weitz}{\textbf{WEITZ}}

\newcommand{\prob}{\textbf{Pr}}
\newcommand{\ex}{\mathbb{E}}
\newtheorem{definition}{Definition}
\newcommand{\opt}{\text{OPT}}
\usepackage{dsfont}
\usepackage{natbib}
\setcitestyle{authoryear}
\usepackage{comment}
\newtheorem{claim}{Claim}
\newtheorem*{lemmano}{Lemma}
\newtheorem{corollary}{Corollary}
\usepackage{xcolor}
\title{Contextual Learning for Stochastic Optimization}
\author{Anna Heuser\footnote{Institute of Computer Science and Lamarr Institute for Machine
Learning and Artificial
Intelligence, University of Bonn, Germany. Email:  aheuser1@uni-bonn.de} \and Thomas Kesselheim\footnote{Institute of Computer Science and Lamarr Institute for Machine
Learning and Artificial
Intelligence, University of Bonn, Germany. Email:  thomas.kesselheim@uni-bonn.de}}
\newcommand{\rwd}{V^*}
\newcommand{\lwd}{V'}

\newcommand{\contextdist}{X}
\newtheorem{theorem}{Theorem}

\newtheorem{lemma}[theorem]{Lemma}

\begin{document}
\maketitle
\begin{abstract}
   Motivated by stochastic optimization, we introduce the problem of learning from samples of contextual value distributions. A contextual value distribution can be understood as a family of real-valued distributions, where each sample consists of a context $x$ and a random variable drawn from the corresponding real-valued distribution $D_x$.  By minimizing a convex surrogate loss, we learn an empirical distribution $D'_x$ for each context, ensuring a small Lévy distance to $D_x$. We apply this result to obtain the sample complexity bounds for the learning of an $\epsilon$-optimal policy for stochastic optimization problems defined on an unknown contextual value distribution. The sample complexity is shown to be polynomial for the general case of strongly monotone and stable optimization problems, including Single-item Revenue Maximization, Pandora's Box and Optimal Stopping. 
 
\end{abstract}

\setcounter{tocdepth}{1} 

\section{Introduction}

In many real-world scenarios, decisions must be made under uncertainty, for example when optimizing pricing strategies for revenue maximization, minimizing production costs, or managing financial risk. Stochastic optimization provides a powerful framework for making optimal decisions in such uncertain environments.

A standard assumption throughout the literature is that the underlying probability distributions are known. A common explanation is that one can use past observations to gain this knowledge. In a recent line of research, this has been captured more explicitly. A typical example is the samples framework: Instead of assuming full knowledge of the probability distributions, we only see a bounded number of samples of each of them.

A major weakness of this framework is that the samples have to come from the same distributions. But not all days are the same: The distribution (capturing, for example, valuations or demand) will usually depend on the season, the day of the week, and alike. The sample framework neglects any such contextual information.

In this paper, we address this weakness by introducing a contextual model of probability distributions and ask the question:
\begin{quote}
    How can we use samples obtained in different contexts to (approximately) solve stochastic optimization problems?
\end{quote}

One should mention that indeed every regression problem can be understood as such a contextual learning problem. However, using regression, we only learn expectations of underlying distributions, which is not enough for stochastic optimization problems.\\
Besides regression, contextual models have been introduced in various fields before. They are best known in the field of contextual multi-armed bandits, but recently there have been a few approaches to contextual Pandora's Box \citep{contextualpandora}, contextual Dynamic Pricing \citep{tullii2024improved} and contextual Revenue Maximization \citep{devanur2016sample}. However the definitions of contextual distributions and the used techniques are very specific for the individual problems in all these examples. Our approach is the first to give generalized results for contextual Stochastic Optimization Problems. 

Our main finding is that (in benign settings) it suffices to minimize a (convex) surrogate loss function on the samples to learn a distribution that is close to the true one in Lévy distance for almost every context. As a consequence, a polynomial number of samples from different contexts suffices for many stochastic optimization problems.

\subsection{A Motivating Problem: Single-Buyer Single-Item Revenue Maximization}

As a motivating example, consider the problem of maximizing revenue when selling a single item to a single buyer. That is, we would need to find a price $p$ such that $p \cdot \Pr[y \geq p]$ is maximized, where $y$ is the buyer's value for the item. In our contextual model, we assume that the buyer's value $y$ is determined as $y = f(v,x)$, where $v \sim \rwd$ is a random weight vector, $x \sim \contextdist$ is a random context vector, and $f$ is a known function. The function $f$ is generally any convex and Lipschitz function. An example to keep in mind is a linear function $f(v, x) = \langle v, x \rangle$. Then a context vector $x \in \{0, 1\}^d$ would activate or deactivate certain additive components of the distribution.

When setting the price, the function $f$ and the context $x$ are known to us. Also knowing the distribution $\rwd$, finding such a price would be straightforward. However, we have only sample information. That is, we get to see $(x_1,y_1), \ldots, (x_m,y_m)$, where $y_j = f(v_j,x_j)$, $x_j \stackrel{\mathrm{i.i.d.}}{\sim} \contextdist$, and $v_j \stackrel{\mathrm{i.i.d.}}{\sim} \rwd$. Importantly, the underlying weight vectors $v_1, \ldots, v_m$ are not revealed to us. The key question is then to arrive at a good decision for a given context $x$, although most likely it is different from all context $x_1, \ldots, x_m$ we have seen samples for because there can be exponentially or even infinitely many different such vectors.

Moving to multiple buyers (and more complex optimization problems), we have one weight distribution $\rwd_i$ for every buyer and the respective draws are independent. The context is (without loss of generality) the same for all buyers. Our general, non-linear model allows us to capture interesting scenarios. As a simple example, consider $x,v \in \{0, 1\}^2$ and set $f(v, x) = (1 - x_1 v_1) v_2$. Then, we have $f(v, x) = v_2$ if $x_1 = 0$ or $v_1 = 0$ and $f(v, x) = 0$ otherwise. This way, we can express that certain buyers are not present at all depending on the context.

\subsection{Main Results and Techniques}
We give a general technique to learn contextual distributions based on samples. We show that polynomially many samples suffice to get within an additive $\epsilon$ error for Single-item Revenue Maximization, Optimal Stopping, and Pandora's Box.

In more detail, inspired by the squared loss for linear regression we introduce the capped squared loss of distribution $\lwd$ with respect to distribution $\rwd$ defined as  \[L^x\left (\lwd\right )=\ex_{v^*\sim \rwd}\left [\sum_{c\in C_\epsilon}\left (\ex_{v\sim \lwd}[\max\{c,f(v,x)\}]-\max\{c,
f(v^*,x)\}\right )^2\right ],\] where $C_\epsilon$ is a suitable discretization of real numbers. It plays a central role in our analysis: We show that on the one hand we can find a distribution efficiently that approximately minimizes the loss based on samples. On the other hand, we prove that any distribution of approximately minimal loss needs to be close to the true one in Lévy distance.

To show learnability, we use that the loss function is both convex and Lipschitz. It allows us to draw on established results from convex learning theory. In particular, we show that minimizing the regularized empirical loss leads to a distribution that also has a small true loss. Importantly, this minimization problem remains computationally tractable because we can restrict our attention to distributions that are uniform over a polynomial number of vectors, requiring us to only minimize a convex function over polynomially many dimensions.

Our main contribution is then to establish that a small loss translates directly to a small Lévy distance to the true distribution. For that we estimate the capped expectations of a distribution with small loss. We then show a connection between capped expectation and Lévy-distance. As a consequence, we show that with $O\left (\frac{d}{\epsilon^{16}\delta^2}\right )$ samples we obtain a distribution $\lwd$ fulfilling $d_L(\rwd,\lwd)\leq \epsilon$ with probability at least $1 - \delta$.

This result has important implications on contextual learnability of stochastic optimization problems. We observe that, for many such problems of interest, the reward function of the optimal policy remains largely unchanged when the underlying distribution is perturbed within a small Lévy distance. This robustness property, combined with the fact that these reward functions exhibit strong monotonicity in the sense of \citet{guo2021generalizing}, allows us to derive sample complexity guarantees for stochastic optimization problems. In particular we show that for the problems of contextual Single-item Revenue Maximization, contextual Optimal Stopping and contextual Pandora's Box $O\left (\frac{nd}{\epsilon^{16}\delta^2}\right)$ samples are sufficient to approximate the optimal policy with an error of at most $n\epsilon$ with a probability of at least $1-n\delta$. Here $n$ describes the number of buyers or boxes in the mentioned optimization problems and $d$ refers to the number of dimensions in context vectors. 

Due to its generality, deriving bounds via the Lévy distance is quite lossy. For Optimal Stopping and Pandora's Box, we therefore skip this intermediate step. Instead we rely on the fact that to find an optimal policy for such a problem we do not need to know the full distribution but only the capped expectations. This way, we obtain sample complexities of $O\left (\frac{nd}{\epsilon^{8}\delta^2}\right )$.

\subsection{Related Work and Additional Background}
\textbf{Contextual distributions} have been intensely studied in the field of contextual bandits, a special case of multi-armed bandits, in which the reward is generated via contextual distributions. The only characteristic of the underlying weight distributions, that needs to be approximated in these problems, is the expected value. Our techniques on contextual learning result in a more specific knowledge about the underlying weight distributions. Thus we receive a good approximation of the expected value, but also on other characteristics such as the cumulative distribution function. 
    
Earlier works on contextual multi-armed bandits, such as \citep{pmlr-v32-agarwalb14}, \citep{opt_contextual_bandits} and \citep{NIPS2007_4b04a686}, use classification oracles. More recently the focus has mainly been on regression oracles \citep{simchi2022bypassing, foster2018practical}. The list continues to different variations such as contextual bandits with packing and covering constraints \citep{slivkins2023contextual} or linear contextual bandits \citep{pmlr-v15-chu11a}.

Closer to the problems considered in this paper is contextual dynamic pricing. Here the general assumption is that the value of a buyer is generated by a fixed function on the context. \citet{features} for example study the problem of contextual dynamic pricing with adversarially chosen contexts and a fixed, deterministic function on the context. More recent results assume an added zero-mean noise on the buyers value, independent of the context. 
The zero-mean noise is either Lipschitz \citep{tullii2024improved, chen2021nonparametric,xu2022towards}, its second derivative is bounded \citep{luo2024distribution} or it is log-concave \citep{javanmard2019dynamic}. Our definition of a contextual distribution is stronger and generalizes the common definition of contextual dynamic pricing. In particular our model allows the noise, or the variance of the real-valued vector distributions, to be correlated with the context.

\citet{devanur2016sample} study Revenue Maximization with signals. A signal is just a one-dimensional context and they define a distribution on signals as well as a reward distribution for each signal. The only assumption made is, that the reward distribution of a higher signal stochastically dominates that of a smaller signal. \\
More recently \citet{contextualpandora} give some results on contextual Pandora's Box. In our definition of contextual Pandora's Box we assume that the reward of a box is generated using a function on a weight vector, drawn from an unknown distribution. \citet{contextualpandora} assume, that the fair-cap of a box is generated by a function on a fixed weight vector and the context. This assumption is less realistic and reduces the problem of contextual Pandora's Box to learning one expected value for each box.\\
In the field of \textbf{learning optimal policies for stochastic optimization problems} \citet{guo2021generalizing} give a general sample complexity bound for monotone optimization problems. When looking at problem-specific bounds \citet{correa2022prophet} study the learning of i.i.d. prophet inequalities, \citet{jin2024sample} on general prophet inequalities. The problem of learning Pandora's box has mainly been studied in the more challenging bandit setting \citep{pmlr-v162-gergatsouli22a, gatmiry2024bandit}. Examples of Revenue Maximization with samples are given by \citet{cole2014sample} and \citet{roughgarden2015ironing}, which concentrate on learning a good empirical virtual value via samples. A recent approach on Stochastic Optimization via Semi-Bandit feedback is given by \citet{DBLP:conf/focs/AgarwalGN24}.

\section{Problem Statement}
Let $\rwd$ and $\contextdist$ be distributions of $d$-dimensional vectors with a support in $[0,1]^d$. We call $\rwd$ the weight distribution and $\contextdist$ the context distribution. Additionally let $f:[0,1]^d\times [0,1]^d \rightarrow [0,c_{\max}]$ be the reward function, always known to the learner. The reward of a weight vector $v$ and a context $x$ is then defined by $f(v,x)$. We will call the distribution of rewards the \textbf{contextual value distribution} and denote it by $(\rwd, \contextdist,f)$. A sample of a contextual value distribution is a tuple $(x,f(v,x))$, where $x\sim \contextdist$ and $v\sim \rwd$ independently. Note that the weight vector $v^*$ is not visible in a sample. The only information gained about $v$ is the reward $f(v,x)$.\\
The standard assumption would be $f(v,x)=\langle v,x \rangle$, as is used for example in linear contextual bandits \citep{pmlr-v15-chu11a}. We generalize the linear case by assuming that the reward function is convex and Lipschitz in both variables. That is, we have \[|f(v,x)-f(v',x)|\leq \xi \|v-v'\|_2\] for all $v,v'\in [0,1]^d$ and \[|f(v,x)-f(v,x')|\leq \sqrt{d}\|x-x'\|_2\] for all $x,x'\in [0,1]^d.$
Intuitively this translates to a separate real-value distribution for each context.  When drawing a sample from the contextual value distribution, we draw a context and then a value from the real-value distribution given by the respective context. The weight distribution $\rwd$ and the reward function $f$ define the common structure of the real-value distributions. \\
A contextual stochastic optimization problem is an optimization problem defined on $n$ contextual value distributions as follows: 
\begin{enumerate}[label=\arabic*.]
    \item The weight distributions $\rwd_1,\dots \rwd_n$, one context distribution $\contextdist$ and a reward function $f$ are chosen. 
    \item A context $x$ is drawn from $\contextdist$ and revealed to the player.
     \item The player decides on a policy $\pi_x$.
     \item The weight vectors $v^*_1, \dots , v^*_n$ are drawn from the weight distributions independently.
     \item The player obtains a reward of $\pi_x(f(v^*_1,x), \dots , f(v^*_n,x))$.   
\end{enumerate}
We will concentrate on the setting of unknown weight and context distributions, but assume that we have access to $m$ samples $(x_1, f(v^*_1,x_1)), \dots (x_m, f(v^*_m, x_m))$ from each of the contextual value distributions. 

\textbf{What is the number of samples needed to be able to approximate the optimal policy?}
More precisely let $\pi^*_x$ denote the optimal policy given context $x$ and $\pi(r_1, \dots , r_n)$ denotes the reward of a policy $\pi$ on the realization of values $r_1, \dots , r_n$. Then what is the number of samples needed, such that, when given a context $x\sim \contextdist$, we can decide on a policy $\pi_x$, such that 
\[\ex\left [\pi_x(f(v_1,x), \dots f(v_n,x))\right ]\geq \ex\left [\pi^*_x(f(v_1,x), \dots f(v_n,x))\right ]-\epsilon\]
with a probability of at least $1-\delta$? \\
Here the expectation is taken over the random draws of weight vectors, not the random draw of the context, as the context is fixed and known, when deciding for a policy.

\section{Defining a Surrogate Loss Function}
Our approach to learn a distribution $\lwd$ is to define a loss function on the set of all distributions, which quantifies how ``close'' $\lwd$ is to $\rwd$. This function needs to be carefully chosen to ensure that it can be estimated and minimized only based on samples. The difficulty is that the number of different contexts can be huge and we cannot assume to observe the same context multiple times.

In order to define the loss function, first suppose that given a context $x$ we had to predict the capped expectation $\ex_{v\sim \rwd}\left [\max\{c,f( v,x)\}\right ]$ for a fixed $c$. Then it would be natural to perform a least-squares regression, minimizing 
$\frac{1}{m} \sum_{i = 1}^m \left(\ex_{v'\sim \lwd}\left [ \max\{c,f(v',x_i)\}\right ] - \max\{c, y_i\} \right)^2$ or a regularized version thereof. Then, indeed, the minimizer gives us an unbiased estimated of the expectation.

Our true problem requires us to be able to estimate $\ex_{v\sim \rwd}\left [\max\{c,f( v,x)\}\right ]$ for all values of $c$ simultaneously for a given context $x$. Therefore, we add up the respective quadratic error terms for all $c$ from a sufficiently dense set. In our case $C_\epsilon =\{i \cdot \epsilon \mid i\in \{0, \dots, c_{\max}/\epsilon -1\}$ for a suitable chosen $\epsilon$ will give the desired results. Based on this, we define \textbf{capped squared loss} of a distribution $\lwd$ on a sample $(x, y)$ as 
\[\ell\left (\lwd,(x,y)\right )=\sum_{c\in C_\epsilon}\left (\ex_{v\sim \lwd}[\max\{c,f(v,x)\}]-\max\{c,y\}\right )^2.\]

Consequently, on a set of samples $S=\{(x_1, y_1), \ldots, (x_m, y_m)\}$, the empirical loss is defined as the average of the individual losses of the samples
\[
L(\lwd, S) = \frac{1}{m} \sum_{i=1}^m \ell(\lwd, (x_i,y_i)) = \sum_{c \in C_\epsilon} \frac{1}{m} \sum_{i=1}^m \left (\ex_{v\sim \lwd}[\max\{c,f(v,x)\}]-\max\{c,y\}\right )^2.
\]
Note that this last step is only possible because in the loss definition we are taking the sum over all $c$, rather than, for example the maximum. Analogously we define the true loss with respect to distribution $\rwd$ as 
\[L(\lwd)=\ex_{x\sim X}\left [\ex_{v^*\sim \rwd}\left[\sum_{c\in C_\epsilon}\left (\ex_{v'\sim \lwd}\left [\max\{c,f(v',x)\}\right ]-\max\{c,f(v^*,x)\}\right )^2\right]\right ]\] and the true loss for a fixed context $x$ as 
\[L^x(\lwd)=\ex_{v^*\sim \rwd}\left[\sum_{c\in C_\epsilon}\left (\ex_{v'\sim \lwd}\left [\max\{c,f(v',x)\}\right ]-\max\{c,f(v^*,x)\}\right )^2\right].\] 
Note that in general the capped squared loss, as well as the true loss, takes on a positive value and, even for the real underlying distribution $\rwd$, it is not zero. Instead $L^x(\rwd)$ can be understood as the sum of variances of the random variables $\max\{c, f(v, x)\}$ for all $c\in C_{\epsilon}$.\\
However $\rwd$ minimizes the true loss and, although the quality of a distribution $\lwd$, given by $|L(\rwd)-L(\lwd)|$ is unknown, the quality improves with decreasing true loss. Additionally the capped squared loss is by design convex.

\section{Efficiently Learning a Distribution with Small Loss from Samples}
First of all, we will make use of a result from convex learning to learn a distribution with a small true loss using samples. In particular the learned distribution will be a uniform distribution with bounded support. However this restriction does not limit our results, as the true loss can still be arbitrarily close to the true loss of $\rwd$.

 \begin{theorem}\label{max_expectation:high-prob}
 Using $m\geq \frac{32d\xi^2c_{\max}^4}{\epsilon^4\delta^2}$ samples we can learn a weight distribution $\lwd$, such that with probability of at least $1-\delta$ we have 
        \[L^x(\lwd)\leq L^x(\rwd)+2\epsilon,\] if $x$ is a context drawn from $X$.
    \end{theorem}
In the proof we use known results from convex learning to bound the sample complexity for learning a distribution with bounded support and small true loss. We then show that any distribution can be approximated by a distribution with bounded support. 
\subsection{Restriction to Bounded Support}
We restrict our attention to distributions with bounded support, as they can efficiently be learned via convex learning. But in fact for any distribution $\rwd$ there exists a distribution with finite support that is close to $\rwd$ with respect to the true loss. Therefore the true loss of the learned distribution will also be close to the true loss of the real underlying weight distribution $\rwd$.
\begin{lemma}\label{gen_distr}
    For any distribution $\rwd$ over weight vectors $v\in [0,1]^d$ there exists a distribution $V$ with a support of at most $2\frac{c_{\max}^3}{\epsilon^2}\left (d\log\left (dc_{\max}\right ) + (d+1)\log \left (2/\epsilon\right )\right ) $ vectors, such that for any context $x\in [0,1]^d$ we have $L^x(V) \leq L^x(\rwd) + \epsilon$.
\end{lemma}
The full proof is given in Section \ref{proof_gen_distr}. The idea is that the true loss can be bounded by $\frac{c_{\max}}{\epsilon}$ times the squared distance of the capped expectation. Hence we can focus on finding a distribution $V$, such that 
\[\left |\ex_{v\sim V}\left [\max\{c,f(v_i,x)\}\right ]-\ex_{v^*\sim \rwd}\left [\max\{c,f(v^*,x)\}\right ]\right|\] is small.
We will use the fact that an empirical value, given by a number of samples, is equal to the expected value of the uniform distribution on these samples. Therefore it suffices to show, that, given $n\geq2\frac{c_{\max}^3}{\epsilon^2}\left (d\log\left (dc_{\max}\right )+(d+1)\log \left (2/\epsilon\right )\right ) $  samples from $\rwd$, the probability that these samples approximate the capped expectation is positive for any $x$ and any $c$. For that we can use Hoeffding's inequality and a union bound over discretized sets of constants $c$ and contexts $x$. 
\subsection{Learning from Samples}

In convex learning the goal is to learn a hypothesis which minimizes a true loss function via samples. The existing results, given by \citet{JMLR:shalev-shwartz}, only require the loss function to be convex and Lipschitz, and the hypothesis space to be closed, convex and bounded. Let $\mathcal{V}^k$ denote the set of all uniform distributions with a support of exactly $k$ vectors in $[0,1]^d$. When restricting our problem to learning a distribution in $\mathcal{V}^k$, we have a closed, convex and bounded hypothesis space. In particular note that any $V\in \mathcal{V}^k$ can be expressed as a vector in $[0,1]^{kd}$ and therefore $\|V\|_2\leq \sqrt{kd}$. Additionally the following Lemma gives an upper bound on the Lipschitz constant for the capped squared loss. 

 \begin{lemma}\label{lipschitz-bound}
            The Lipschitz constant of the capped squared loss restricted to distributions in $\mathcal{V}^k$ can be bounded by $\frac{2c_{\max}^2}{\epsilon \sqrt{k}}\xi$. 
        \end{lemma}
        Note that $\xi$ denotes the Lipschitz constant of the reward function, i.e. $|f(v,x)-f(v',x)|\leq \xi \|v-v'\|_2$ for all $v,v'\in [0,1]^d$.
        \begin{proof} 
        To show Lipschitz continuity it is sufficient to upper bound the norm of the gradient. Let $\rwd\in [0,1]^{kd}$ be a distribution and $v_{i,j}$ denotes the $j$-th entry of the $i$-th vector in this distribution. Now we consider the partial derivative
            \begin{align*}
				\pd{}{v_{i,j}}\ell(V,(x,y))&= \sum_{c\in C_\epsilon}\left (\frac{2}{k}\sum_{i=1}^{k}{\max\{c, f(v_i,x)\}}-y\right )\frac{1}{k}\pd{}{v_j}(f(v,x))\mathbbm{1}\{c\leq f(v_i,x)\} \\
                &\leq \sum_{c\in C_\epsilon}{\frac{2}{k} \pd{}{v_j}(f(v,x))\max_{v\in \rwd}\max\{c,f(v,x)\}}\leq \frac{2}{k}\pd{}{v_j}(f(v,x))c_{\max}^2/\epsilon
			\end{align*} 
            and as $f$ is $\xi$-Lipschitz it follows that $\|\nabla \ell\|_2\leq \frac{2c_{\max}^2}{\epsilon \sqrt{k}}\xi$.
        \end{proof}
So changing a distribution slightly results in a small change in the capped squared loss. In particular the Lipschitz constant decreases with increasing $k$. This matches the intuition, that a smaller probability mass on a particular vector in the support results in a smaller impact of that vector on the loss. 
\begin{lemma}\label{max_expectation}
For any $k\in \mathbb{N}$ and any $\alpha$ we can learn a weight distribution $\lwd$, such that 
\[\ex_S\left [L(\lwd)\right ]\leq \min _{V\in \mathcal{V}^k}L(V)+\alpha\] using a set of samples $S$ with $|S|\geq \frac{32d\xi^2c_{\max}^4}{\epsilon^2\alpha^2}$.
\end{lemma}
We give a short introduction on convex learning and the full theorem stated by \citet{JMLR:shalev-shwartz} in Section \ref{convex_learning}. The learned distribution $\lwd$ is obtained by finding the distribution, which minimizes the loss on the sample set $S$ and an additive regularization term. \\
The sample complexity depends only on $\alpha$, the Lipschitz constant of the loss and an upper bound on $\|V\|_2$. In particular note that it does not depend on the size of the distributions $k$. Although both the Lipschitz constant and the norm of a distribution in $\mathcal{V}^k$ depend on $k$, the dependencies cancel out in the sample complexity.  However using these techniques we can only learn a distribution with bounded support. 

Using our knowledge about distributions with bounded support and Markov's inequality we can then get a high probability result for the true loss for a fixed context to the true loss of $\rwd$.

\begin{proof}[Proof of Theorem \ref{max_expectation:high-prob}]
Let $S$ be a set of samples of size $\frac{32d\xi^2c_{\max}^4}{\epsilon^4\delta^2}$. Then by using Lemma \ref{max_expectation} we know that we can learn a distribution $V$ with 
    \[\ex_S\left [L(\lwd)\right ]\leq \min _{V\in \mathcal{V}^k}L(V)+\epsilon\delta\] for any $k$. From Markov's inequality it then follows, that 
    \[L^x(\lwd)\leq \min _{V\in \mathcal{V}^k}L^x(V)+\epsilon\] for some $x\sim X$ with probability of at least $1-\delta$. If we choose $k$ to be at least $2\frac{c_{\max}^3}{\epsilon^2}\log\left (2\frac{(2dc_{\max})^d}{\epsilon ^{d+1}}\right ) $, then it immediately follows that 
    \[L^x(\lwd)\leq L^x(\rwd)+2\epsilon.\] 
\end{proof}
Note again that the true loss of $\rwd$ can be arbitrarily high and therefore the value of the true loss itself does not give us any information on the quality of our learned distribution. However $\rwd$ still minimizes the true loss and the distance of $L^x(\lwd)$ to $L^x(\rwd)$ is what measures the quality of the learned distribution.

\section{Absolute Difference of Capped Expectation}

Now we will analyze the connection between the absolute difference between $\lwd$ and the real underlying distribution $\rwd$ in the true loss  and the absolute difference in the capped expectation. In particular we will show that a difference in the true loss of $\epsilon$ for some $x$ implies a difference in the capped expectation of $\sqrt{\epsilon}$ for any $c\in C_\epsilon$. 
\begin{lemma}\label{lemma_capped-expectation}
If for some $x$ we have $L^x(\lwd)\leq L^x(\rwd)+\epsilon$, then \[\left |\ex_{v'\sim \lwd}\left [\max\{c,f(v',x)\}\right ]-\ex_{v\sim \rwd}\left [\max\{c,f(v,x)\}\right ]\right |\leq \sqrt{\epsilon}+\epsilon\] for all $c$. 
\end{lemma}
Intuitively this shows that the capped squared loss lets us learn many expected values at once, resulting in a learned distribution which approximately minimizes all these expected values. 
Note that overall this implies the following result on the sample complexity sufficient for a small difference in capped expectation. 
\begin{corollary}
     Using $m\geq O\left (\frac{d\xi^2c_{\max}^4}{\epsilon^8\delta^2}\right )$ samples we can learn a weight distribution $\lwd$, such that with probability of at least $1-\delta$ we have 
\[\left |\ex_{v'\sim \lwd}\left [\max\{c,f(v',x)\}\right ]-\ex_{v\sim \rwd}\left [\max\{c,f(v,x)\}\right ]\right |\leq \sqrt{\epsilon}+\epsilon,\]if $x$ is a context drawn from $X$.
\end{corollary}

 The following lemma generalizes the key property of the capped squared loss.
\begin{lemma}\label{lin_reg_lemma}
    Let $y$ be a random variable. Then for any $z,z'$ we have 
    \[\left (z-\ex[y]\right )^2-\left (z'-\ex[y]\right )^2= \ex\left [ (z-y)^2\right ]-\ex\left [(z'-y)^2\right ].\]
   
\end{lemma}
    \begin{proof}For any $z,z'$ and any random variable $y$ we have 
    \[\begin{aligned}
        &\left(z-\ex[y]\right )^2-\left (z'-\ex[y]\right )^2 \\
        =&z^2-2z\ex[y]+\ex[y]^2-{z'}^2+2z'\ex[y]-\ex[y]^2 .   
    \end{aligned}\]
    We observe that $\ex[y]^2$ cancels out. Then we can use linearity of expectation to obtain 
    \[\begin{aligned}
        &\ex \left [z^2-2zy+y^2-{z'}^2+2z'y-y^2\right ]= \ex\left [(z-y)^2\right ]-\ex\left [(z'-y)^2\right ].
    \end{aligned}\]
    
    \end{proof}
In particular, it follows that $(z-\ex[y])^2=\ex[(z-y)^2]-\ex[(\ex[y]-y)^2]$. Now we are ready to prove our main result on the distance of capped expectations.  
\begin{proof}[Proof of Lemma \ref{lemma_capped-expectation}]
From Lemma \ref{lin_reg_lemma} we have 
\[\begin{aligned}
&\left(\ex_{v'\sim \lwd}\left [\max\{c,f(v',x)\}\right ]-\ex_{v^*\sim \rwd}\left [\max\{c,f(v^*,x)\}\right ]\right)^2\\
=&\ex_{v^*\sim \rwd} \left [\left (\ex_{v'\sim \lwd}\left [\max\{c,f(v',x)\}\right ]-\max\{c,f(v^*,x)\}\right )^2\right ]\\
-&\ex_{v^*\sim \rwd} \left [\left (\ex_{v\sim \rwd}\left [\max\{c,f(v,x)\}\right ]-\max\{c,f(v^*,x)\}\right )^2\right ]
\end{aligned}\] for any $c$ and any $x$. Note that, when taking the sum over all $c\in C_\epsilon$ on both sides, this means that the distance of $L^x(\lwd)$ and $L^x(\rwd)$ is exactly equal to the sum of squared distances of the capped expectation for all $c\in C_\epsilon$. 

  For the true loss it then follows  \[\begin{aligned}
   & \epsilon \geq L^x(\lwd)-L^x(\rwd)\\
    =&\sum_{c\in C_\epsilon}\ex_{v^*\sim \rwd} \left [\left (\ex_{v'\sim \lwd}\left [\max\{c,f(v',x)\}\right ]-\max\{c,f(v^*,x)\}\right )^2\right ]\\
    -&\sum_{c\in C_{\epsilon}}\ex_{v^*\sim \rwd} \left [\left (\ex_{v\sim \rwd}\left [\max\{c,f(v,x)\}\right ]-\max\{c,f(v^*,x)\}\right )^2\right ]\\
    =& \sum_{c\in C_\epsilon}\left(\ex_{v'\sim \lwd}\left [\max\{c,f(v',x)\}\right ]-\ex_{v^*\sim \rwd}\left [\max\{c,f(v^*,x)\}\right ]\right)^2.
    \end{aligned}\]
    Therefore fore any $c\in C_\epsilon$ \[\left |\ex_{v'\sim \lwd}\left [\max\{c,f(v',x)\}\right ]-\ex_{v\sim \rwd}\left [\max\{c,f(v,x)\}\right ]\right |\leq \sqrt{\epsilon}.\]
    Now for any $c'\in [0,c_{\max}]$ there exists some $c\in C_\epsilon$ such that $|\max\{c,f(v,x)\}-\{c',f(v,x)\}|\leq \epsilon$. In particular the maximum can only increase when increasing $c$. Overall this proves the Lemma.
\end{proof}
Note that this does not hold when comparing two distributions $\lwd$ and $V$, none of which is the real underlying distribution. For these, $L^x(\lwd)\leq L^x(V)\leq \epsilon$ would only imply that 
\[\begin{aligned}
&\sum_{c\in C_\epsilon}\left(\ex_{v'\sim \lwd}\left [\max\{c,f(v',x)\}\right ]-\ex_{v^*\sim \rwd}\left [\max\{c,f(v^*,x)\}\right ]\right)^2\\
&-\sum_{c\in C_\epsilon}\left(\ex_{v\sim V}\left [\max\{c,f(v,x)\}\right ]-\ex_{v^*\sim \rwd}\left [\max\{c,f(v^*,x)\}\right ]\right)^2
\end{aligned}
\]
is at most $\epsilon$.
Therefore we would not receive any good bound on the capped expectation for any fixed $c$.

\section{Lévy Metric}
In the previous section we have seen how to learn the expected maximum value for any constant. We will see now that this implies that the learned distribution is close to the original distribution in terms of the cumulative distribution function. We will measure the distance of the distribution in terms of the Lévy metric. For that we want to look at the reward distributions induced by a context and a weight distribution separately. So let $D_x^V$ denote this reward distribution. The cumulative distribution function of $D_x^V$ is given by 
    $F_x^{V}(z)=\prob{\left (f(v,x)\leq z\right )}$. Now we are ready to give a definition for the Lévy metric. 
    \begin{definition}
    Let $D$ and $D'$ be two real-valued distributions. Now the Lévy metric is defined as \[d_L(D,D'):=\inf\{\epsilon: F^D(z-\epsilon)-\epsilon \leq F^{D'}\leq F^D(z+\epsilon)+\epsilon \text{ for all }z\}.\]
    \end{definition}
So the Lévy metric is a way of measuring the maximal distance of two cumulative distribution functions while allowing horizontal as well as vertical shifts. We will now show that our results from the previous sections imply a small Lévy distance. 
\begin{theorem}\label{Levy:bound}
    Let $\rwd$ and $\lwd$ be two distributions of vectors and a context vector $x$ such that $|\ex_{v\sim \rwd}\left [\max\{c,f(v,x)\}\right ]-\ex_{v'\sim \lwd}\left [\max\{c,f(v',x)\}\right ] |\leq \epsilon$ for all constants $c$.\\
    Then we have $d_L(D_x^{\rwd},D_x^{\lwd})\leq \sqrt{2\epsilon}$. 
\end{theorem}

\begin{proof}
    Let us consider the indicator function $\mathbbm{1}(f(v,x) \geq c)$ first. This function can be roughly described by the difference of two maxima. In fact we have 
    \[ \begin{aligned}&\max\{0,
    f(v,x)-c\}-\max\{0,f(v,x)-(c+1) \}\leq \mathbbm{1}(f(v,x) \geq c)\\ \leq&\max\{0,f(v,x)-(c-1)\}-\max\{0,f(v,x) -c\}.\end{aligned} \]
    So $\max\{0,
    f(v,x)-c\}-\max\{0,f(v,x)-(c+1) \}$ is at most 1 and if $f(v,x)<c$ it is always 0. Analogously $\max\{0,f(v,x)-(c-1)\}-\max\{0,f(v,x) -c\}$ is at least 0 and if $f(v,x)\geq c$ it is 1. By scaling this translates to the following result for any $\alpha>0$.
    \[ \begin{aligned} & \frac{1}{\alpha}(\max\{0,f(v,x)-c\}-\max\{0,f(v,x)-(c+\alpha) \})\leq \mathbbm{1}(f(v,x)\geq c)\\ \leq &\frac{1}{\alpha}(\max\{0,f(v,x)-(c-\alpha)\}-\max\{0,f(v,x) -c\}). \end{aligned}\]
    We can directly use this to bound the Lévy distance. We have 
   \[ \begin{aligned}
        &\prob_{v\sim \rwd}(f(v,x) \geq c)=\ex_{v\sim \rwd}\left [\mathbbm{1}(f(v,x) \geq c)\right ]\\
        \leq &\ex_{v\sim \rwd}\left [ \frac{1}{\sqrt{2\epsilon}}(\max\{0,f(v,x)-(c-\sqrt{2\epsilon})\}-\max\{0,f(v,x) -c\})\right ]\\
        \leq  &\ex_{v'\sim \lwd}\left [ \frac{1}{\sqrt{2\epsilon}}(\max\{0,f(v',x)-(c-\sqrt{2\epsilon})\}-\max\{0,f(v',x)\ -c\}+2\epsilon)\right ]\\
        \leq &\ex_{v'\sim \lwd}\left [\mathbbm{1}(f(v',x) \geq c-\sqrt{2\epsilon})\right ]+\sqrt{2\epsilon}=\prob_{v'\sim \lwd}(f(v',x)\geq c-\sqrt{2\epsilon})+\sqrt{2\epsilon}.
    \end{aligned}\]
  Analogously we can show that 
  \[\prob_{v\sim \rwd}(f(v,x) \geq c)\geq \prob_{v'\sim \lwd}(f(v',x)\geq c+\sqrt{2\epsilon})-\sqrt{2\epsilon}\] 
  and as $F^{\rwd}_x(c)=1-\prob_{v\sim \rwd}(f(v,x) \geq c)$ this finally concludes the proof. 
\end{proof}
This then implies the following bound on the sample complexity to learn a distribution with small Lévy distance. 
\begin{corollary}\label{sample-complexity}
     Using $m\geq O\left (\frac{d\xi^2c_{\max}^4}{\epsilon^{16}\delta^2}\right )$ samples we can learn a weight distribution $\lwd$, such that with probability of at least $1-\delta$ we have \[d_L(D_x^{\rwd},D_x^{\lwd})\leq \epsilon\] if $x$ is a context drawn from $X$.\end{corollary}
An analysis of the relationship of different probability metrics is given by \citet{gibbs2002choosing}. In particular they show a linear relationship between the Lévy metric and the Wasserstein metric.
\begin{definition}
    Let $D$ and $D'$ be two real-valued distributions. Then the Wasserstein metric, denoted by $d_W(D,D')$ is the infimum over all $\epsilon$, for which there is a joint distribution $\tilde{D}$ over pairs $z,z'$ such that the marginal distributions of $z$ and $z'$ are $D$ and $D'$ respectively and it holds that $\ex[|v-v'|]\leq \epsilon$.
\end{definition}
We have $d_W(D,D')\leq 4d_L(D,D')$ and therefore the sample complexity to obtain an $\epsilon$-approximation with respect to the Wasserstein metric increases only by a constant factor, compared to the sample complexity for the Lévy metric. For some of the following applications the existence of a coupling, as given in the definition of the Wasserstein metric, will lead to a more intuitive analysis. 
\section{Applications}
Now we want to come back to contextual stochastic optimization problems. We define several well known stochastic optimization problems in their contextual variant. We will give a general bound on the sample complexity for the class of monotone and stable optimization problems, which all the mentioned optimization problems fall in to. Additionally we obtain an improved bound for contextual Pandora's Box and contextual Optimal Stopping with a more specific analysis using our results from Maximum Regression. 

\subsection{Lévy-stability}
For this section we will consider stochastic optimization on real-valued distributions. As any context induces such a distribution in the setting of contextual value distributions, this will then transfer to results for any reward distribution induced by a context.

Let $(D,A)$ be an instance of a stochastic optimization problem 
where $A$ denotes the set of all feasible policies and $D=D_1\times \dots \times D_n$. We will refer to the stochastic optimization problem itself as $A$. The reward of a policy $\pi\in A$ on values $r_1, \dots , r_n$ is denoted by $\pi(r_1, \dots , r_n$).
We will use $\mathcal{R}_D(\pi)$ to denote the expected reward of policy $\pi$ on an instance $(D_1, \dots, D_n,A)$ and the optimal policy is denoted as $\pi^*_D$ .
Following the definition given by \citet{guo2021generalizing} a strongly monotone problem is defined as follows. 
\begin{definition}
    A problem $A$ is \textbf{strongly monotone}, if for any $D$ and $D'$, such that $D$ stochastically dominates $D'$:
    \[\mathcal{R}_D(\pi^*_{D'})\geq\mathcal{R}_{D'}(\pi^*_{D'})\]
\end{definition}
So running the optimal policy with respect to $D'$ on distribution $D$ gets at least the optimal reward of an instance defined on $D'$, if $D$ stochastically dominates $D'$.

We also want to establish a notion of stability with respect to the Lévy distance of two distributions. 

\begin{definition}
A problem $A$ is \textbf{$\gamma$-stable}, if for any $D$ and $D'$, such that $d_L(D_i,D'_i)\leq \epsilon $ for all $i$:
\[\mathcal{R}_D(\pi^*_D)\geq \mathcal{R}_{D'}(\pi^*_{D'})-\gamma\epsilon\]
\end{definition}
Intuitively this means, that if we change a distribution slightly, the change in the optimal expected reward on that distribution is bounded. This is the case for many known optimization problems as for example Revenue Maximization and Optimal Stopping.  

\begin{theorem}
    Let $A$ be a strongly monotone and $\gamma$-stable stochastic optimization problem defined on $n$ unknown contextual value distributions 
    $(\rwd_1,X,f), \dots , (\rwd_n, X,f)$. Given at least $O\left (\frac{d\xi^2c_{\max}^4}{\epsilon^{16}\delta^2}\right )$ samples from each contextual value distribution and $x\sim X$, we can learn a policy $\pi'_x\in A$, such that 
    \[\mathcal{R}_{D^{\rwd}_x}(\pi'_x)\geq \opt\left (D^{\rwd}_x \right )-2\gamma\epsilon.\]
    with a probability of at least $1-n\delta$.\end{theorem}
\begin{proof}
Given Corollary \ref{sample-complexity} we can learn distributions $\lwd_1, \dots, \lwd_n$ such that $d_L\left (D_x^{\rwd_i},D_x^{\lwd_i}\right ) \leq \epsilon$ for all $i$ with a probability of at least $1-n\delta$. As a reminder $D_x^V$ denotes the real-valued distribution induced by context $x$ and weight distribution $V$. For simplicity let $D=D^{\rwd_i}_x$ and $E=D_x^{\lwd_i}$.  Using these distributions we want to define a $2\gamma \epsilon$-optimal policy. 
For that we introduce some slightly shifted distributions. So for a distribution $E$ with cumulative distribution function $F$ let $E^{+\epsilon}$ denote the distribution with cumulative distribution function $F^{+\epsilon}(z)=\min \{F(z+\epsilon)+\epsilon,1\}$. So intuitively $E^{+\epsilon}$ takes on slightly smaller values than $E$. In particular $E$ stochastically dominates $E^{+\epsilon}$. Now in a strongly monotone and $\gamma$-stable stochastic optimization problem the optimal policy on $E^{+\epsilon}$, where $E$ has small Lévy distance to $D$ is $\gamma$-optimal. 
    The proof just follows from monotonicity and stability. Obviously $D$ stochastically dominates $E^{+\epsilon}$ and $d_L(D,E^{+\epsilon})\leq 2\epsilon$. Hence 
    \[\begin{aligned}\mathcal{R}_D(\pi^*_{E^{+\epsilon}})\geq \mathcal{R}_{E^{+\epsilon}}(\pi^*_{E^{+\epsilon}})
    \geq \mathcal{R}_D(\pi^*_D)-2\gamma\epsilon
    \end{aligned}\]
Therefore the optimal policy on the shifted variants of $D_x^{\lwd_1}, \dots , D_x^{\lwd_n}$ is $2\gamma\epsilon$-optimal with a probability of  $1-n\delta$, which concludes the proof. 
 \end{proof}
We will show that Pandora's Box, Optimal Stopping and Single-Parameter Revenue Maximization are all strongly monotone and stable optimization problems. Theorem \ref{Levy:bound} gives a high-probability upper-bound on the Lévy distance of the reward distributions given by $\lwd$ and $\rwd$. Therefore we can then find a good approximation of the optimal policy with high probability. 

As an example we want to consider the problem of Single-item and Single-buyer Revenue Maximization. The stability for the general case of Single-item Revenue Maximization has already been shown as we will discuss later. But the analysis of stability in the Single-buyer case gives a good example for stability with a direct proof via the Lévy distance.   In Single-buyer Revenue Maximization any policy can be understood as a price $p$ for the item.  Naturally the buyer with a value of $z$ buys the item, if $z\geq p$, from which we obtain a revenue of $p$. Otherwise the item is not sold and the revenue is 0. The optimal policy would be a price $p$ which maximizes $p\Pr[z\geq p]$.
    \begin{lemma}
        Single-buyer Revenue Maximization is $(c_{\max}+1)$-stable. 
    \end{lemma}
\begin{proof}
    Let $p^*$ be the optimal price with respect do $D$ and let $E$ be a distribution with $d_L(D,E)\leq \epsilon$. Remember that we hence have $F_E(z)\leq F_D(z+\epsilon)+\epsilon$. 
    Then \[
     \max_p\{p\Pr_E(z\geq p)\} \geq (p^*-\epsilon)\Pr_E(z\geq p^*-\epsilon)\] as choosing the price $p^*-\epsilon$ can not be better than the optimal price. Re rewrite this in terms of the cumulative distribution function. 
    \[ (p^*-\epsilon)(1-F_E(p^*-\epsilon))\geq (p^*-\epsilon)(1-F_D(p^*)-\epsilon)\]
    as $E$ and $D$ have small Lévy distance. This is then greater or equal to 
    \[p^*\Pr_D(z\geq p^*)-p^*\epsilon-(1-F_D(p^*))\epsilon\geq p^*\Pr_D(z\geq p^*)-(c_{\max}+1)\epsilon,
    \]
    which proves stability.
\end{proof}
\subsection{Contextual Single-item Revenue Maximization}

Let there be an auctioneer with a single item for sale and $n$ bidders who are interested in the item. For any bidder $i$ there exists an underlying contextual value distribution $(V_i, X,f)$. The value $f(v_i,x)$ of bidder $i$ for the item is private information not known to the other bidders and the auctioneer. The context $x$ however is known to the auctioneer. 
The auctioneer is restricted to DSIC auctions. Or in other words, the mechanism must ensure, that the utility of any bidder is maximized by bidding their true value. The goal is to design a revenue maximizing auction. 

The problem of Revenue Maximization has been studied extensively and there have been prior  results on strong monotonicity \citep{devanur2016sample} and stability  \citep{brustle2020multi}.
\begin{lemma}
Single-item Revenue Maximization is strongly monotone and $O(nc_{\max})$-stable. 
\end{lemma}
Note that it follows, that given $O\left (\frac{d\xi^2c^4_{\max}}{\epsilon^{16}\delta^2}\right )$ samples from each bidder we can find an auction $\pi'$, such that with probability of $1-n\delta$ we get $\mathcal{R}_{D^{V'}_x}(\pi')\geq \mathcal{R}_{D^{V}_x}(\pi^*_{D^V_x})-2nc_{\max}\epsilon=\opt -2nc_{\max}\epsilon$.

\subsection{Contextual Pandora's Box}
    Consider $n$ boxes; each box $i$ has a reward of $f(v_i,x)\in [0,c_{\max}]$, where $v_i\sim \rwd_i$ and $x\sim \contextdist$. The distributions $\contextdist$ and $\rwd_i$ for all $i$ have a support of $X\subseteq [0,1]^d$ and $V_i\subseteq[0,1]^d$. Additionally we have a fixed cost $o_i\in [0,c_{\max}]$ for opening box $i$. Note that the context is fixed for each box and revealed to the algorithm at the beginning. Then in each round the algorithm decides if to take the best observed reward, or to open a new box. The goal is to maximize the reward it gets minus the total cost. \\
    Again it suffices to analyze the problem in its non-contextual variant with distributions $D_1, \dots, D_n$, as any fixed context $x$ induces such a real-valued distribution. 
    It was already shown by \citet{guo2021generalizing}, that the problem of Pandora's box is strongly monotone. However we can also show stability here. 
    \begin{lemma}
        Pandora's box is $4n$-stable. 
    \end{lemma}
    \begin{proof}
        Assume we have distributions $D_1, \dots, D_n$ and $D'_1, \dots , D'_n$ with $d_L(D_i,D'_i)\leq \epsilon$ for all $i$. It then follows that the Wasserstein distance $d_W(D_i,D'_i)\leq 4\epsilon$. We can now use the coupling definition given by the Wasserstein distance. Let $\tilde{D}_i$ be a coupling distribution over pairs $r_i,r'_i$, such that the marginal distributions of $r_i$ and $r'_i$ are given by $D$ and $D'$ and such that $\ex[|r_i-r'_i|]\leq 4\epsilon$. For any fixed $r_i\sim D$ let $\tilde{D}_i(r_i)$ be the distribution of pairs given by $\tilde{D}_i$ conditioned on $r_i$. \\
        At last consider the following policy $\tilde{\pi}$. Upon seeing a value $r_i\sim D_i$ for some $i$ it draws a pair $(r_i,r'_i)$ from $\tilde{D}_i(r_i)$ and acts according to the optimal policy on $D'$, when seeing $r'_i$. 
        Now for $\tilde{\pi}$ the expected reward for Pandora's box on distribution $D$ is given by 
        \[\ex_{(r,r')\sim \tilde{D}}\left [\max_{i \text{ opened by }\tilde{\pi}}r_i-\sum_{i \text{ opened by }\tilde{\pi}}o_i\right ].\]
        Note that we can rewrite the expected value given using our conditional distribution. Also $\tilde{\pi}$ opens a box on realized values $r_1, \dots, r_n$ if and only if the optimal policy on $D$ opens a box on realized values $r'_1, \dots , r'_n$ drawn from $\tilde{D}_1(r_1), \dots, \tilde{D}_n(r_n)$. Therefore the above is equal to 
        \[\begin{aligned}
       & \ex_{r\sim D}\left [\ex_{r'\sim \tilde{D}(r)}\left [\max_{i \text{ opened by OPT}}r_i-\sum_{i \text{ opened by OPT}}o_i\right ]\right ]\\
       =&\ex_{r\sim D}\left [\ex_{r'\sim \tilde{D}(r)}\left [\max_{i \text{ opened by OPT}}(r'_i+r_i-r'_i)-\sum_{i \text{ opened by OPT}}o_i\right ]\right ]\\
       \geq &\ex_{(r,r')\sim \tilde{D}}\left [\max_{i \text{ opened by OPT}}r'_i-\sum_{i \text{ opened by OPT}}o_i-\sum_i|r'_i-r_i|\right ]\\
       \geq& \ex_{(r,r')\sim \tilde{D}}\left [\max_{i \text{ opened by OPT}}r'_i-\sum_{i \text{ opened by OPT}}o_i\right]-4n\epsilon,\end{aligned}\]
       which finishes the proof. 
    \end{proof}
    Again we can conclude an upper bound of $O\left (\frac{d\xi^2c^4_{\max}}{\epsilon^{16}\delta^2}\right )$ on the sample complexity for each box needed to learn an $8n\epsilon$ approximation of the reward of the optimal policy with a probability of at least $1-n\delta$. However we can do better using our results on the capped expectation. 
    
    The optimal policy for Pandora's Box was shown to be the fair-cap policy by \citet{weitzman}.  
    The policy relies on computing the fair cap $\sigma^*_i$ for each box, such that 
    $\ex_{r_i\sim D_i}\left [\max\{0,r_i-\sigma^*_i\} \right ]=o_i$. It then opens boxes in order of decreasing fair cap and stops whenever the highest observed reward exceeds the remaining fair caps.\\
Now consider a single box with distribution $D$ and opening cost $o$.  We want to analyze the fair-cap of a distribution close to $D$ with respect to the capped expectation. 
So let $\sigma(D')$ be the fair cap with respect to some $D'$. Or in other words we define $\sigma(D')$ such that $\ex_{r'\sim D'}\left [\max\{0,r'-\sigma(D')\} \right ]=o$. The following Lemma now gives a bound on the opening costs of a box with distribution $D$ and fair cap $\sigma(D')$. 
\begin{lemma}
    Let $D,D'$ be two distributions with \[\left |\ex_{r\sim D}\left [ \max \{c, r\}\right ]-\ex_{r'\sim D'}\left [ \max \{c, r' \}\right  ]\right | \leq  \epsilon\] for all $c$, then 
    $|\ex_{r\sim D}\left [\max\{0,r-\sigma(D')\}\right ]-o|\leq \epsilon$.
\end{lemma}
\begin{proof}
    As $\sigma(D')$ is the fair cap with respect to $D'$, we can rewrite the opening costs as follows. 
    \[\begin{aligned}
        & \left |\ex_{r\sim D}\left [\max\{0,r-\sigma(D')\}\right ]-o\right|\\
        =&\left |\ex_{r\sim D}\left [\max\{0,r-\sigma(D')\}\right ]-\ex_{r'\sim D'}\left [\max\{0,r'-\sigma(D')\}\right ]\right |\\
         =&\left |\ex_{r\sim D}\left [\max\{\sigma(D'),r\}\right ]-\ex_{r'\sim D'}\left [\max\{\sigma(D'),r'\}\right ]\right |\leq \epsilon.
    \end{aligned}\]
\end{proof}
Now let $\weitz_{D}(\bm{\sigma,o})$ denote the expected reward obtained by the fair cap policy on $n$ boxes with distributions $D_1, \dots D_n$, opening costs $o_1, \dots o_n$ and fair caps $\sigma_1, \dots \sigma_n$. Then the following Lemma is given by \citet{contextualpandora}. Note that the original Lemma is stated using $(o_i-o'_i)^+$ instead of the absolute value. However as $(o_i-o'_i)^+\leq |o_i-o'_i|$, the statement below immediately follows. 
		
\begin{lemma}\label{lemma:pandora_reward}
    Let $\weitz_{D}(\bm{\sigma,o})$ and $\weitz_{D}(\bm{\sigma',o'})$ be the optimal expected reward of two instances of Pandora's Box with the same distribution but different opening costs. Then 
    \[\weitz^x_{D}(\bm{\sigma',o'})\geq \weitz^x_{D}(\bm{\sigma,o}) -\sum_{i=1}^n {|o_i-o'_i|} \]
\end{lemma}	

Hence given an instance of contextual Pandora's box with unknown weight distributions and $m=O\left (\frac{d\xi^2c^4_{\max}}{\epsilon^8\delta^2}\right )$ samples from each distribution, we can learn weight distributions $\lwd_1, \dots \lwd_n$, such that with probability of at least $1-n\delta$ we we can approximate the optimal policy with an error of at most $2n\epsilon$.

\subsection{Contextual Optimal Stopping}
    In the classical problem of Optimal Stopping we are given distributions $D_1, \dots  , D_n$ of $n$ independent random variables. The outcomes $x_i \sim D_i$ for $i \in [n]$ are revealed one-by-one and we have to immediately select/discard $X_i$ with the goal of maximizing the selected random variable in expectation. \\
		The best policy for Optimal Stopping is given by a simple (reverse) dynamic program: always select $X_n$ on reaching it and select $X_i$ for $i < n$ if its value is more than the expected value of this optimal policy on $X_{i+1}, \dots , X_n$. Thus, the optimal policy with expected value Opt can be thought of as a fixed-threshold policy where we select $X_i$ if $X_i > \tau_i$ for $\tau_i$ being the expected value of this policy after i.\\
		In the contextual variant we have $n$ distributions of weight vectors  $\rwd_1,\dots, \rwd_n$, a context distribution $\contextdist$ and a reward function $f:[0,1]^d\times [0,1]^d\rightarrow [0,c_{\max}]$. At first a context $x\sim \contextdist$ is drawn and revealed to us. Then the rewards $f(v_i,x)$ are revealed one-by-one and we have to immediately select/discard $f(v_i,x)$. Again if the weight distributions were known to us, we could calculate the respective reward distributions given context $x$, and apply the optimal threshold policy. This means we would select the $i$-th reward, if $f(v_i,x)\geq \tau_i$, where $\tau_i=\ex_{v_{i+1}\sim \rwd_{i+1}}\left [\max\{\tau_{i+1},f(v_{i+1},x)\}\right ]$.\\
        Now we want to consider the problem of unknown weight distributions. We are first given $m$ samples. Afterwards a context $x\sim \contextdist$ is drawn and we want to apply a policy, which approximates the reward given by the optimal policy with high probability.
        
        The following result on Optimal Stopping with inaccurate distributions was given by \citet{inaccuratepriors}. 
        \begin{lemma}
            If $d_L(D_i,D'_i)\leq \epsilon/4$ for all $i$, then $\mathcal{R}_D(\pi^*_D)\geq \mathcal{R}_{D'}(\pi^*(D'))-2n\epsilon$. 
        \end{lemma} 
        As Theorem \ref{Levy:bound} shows an upper bound on the Lévy metric of distributions learned via Maximum Linear Regression, we immediately receive an upper bound of  $O\left (\frac{d\xi^2c^4_{\max}}{\epsilon^{16}\delta^2}\right )$ on the sample complexity needed to receive an $\epsilon$-approximation of the  optimal for contextual Optimal Stopping with a probability of at least $1-n\delta$. Another way of showing the same sample complexity would be to use the strong monotonicity given by \citet{guo2021generalizing} and $n$-stability given by \citet{inaccuratepriors}. However we can do better by using the structure of the thresholds as the following Lemma shows. 

        \begin{lemma}
        Let $D_1, \dots, D_n$ and $D'_1, \dots , D'_n$ be distributions with 
        \[\left |\ex_{r_i\sim D_i}\left [ \max \{c, r_i\}\right ]-\ex_{r'_i\sim D'_i}\left [ \max \{c, r' _i\}\right  ]\right | \leq  \epsilon\] for all $c$ and all $i$, then  
            \[\mathcal{R}_{D}(\pi^*(D'))\geq \mathcal{R}_{D}(\pi^*(D'))-2n\epsilon.\]
        \end{lemma}
\begin{proof}
    The optimal policy for $D'$, denoted by $\pi^*(D*)$ consists of $n$ thresholds $\tau'_1, \dots, \tau'_n$. We have $\tau'_n=0$ and $\tau'_i=\ex_{r'_{i+1}\sim D'_{i+1}}[\max\{\tau'_{i+1}, r'_{i+1})\}]$. We will now split the proof into to parts. In both parts we will do a backward induction and for that let $\mathcal{R}^{i}_{D}(\pi)$ denote the expected reward of policy $\pi$ on distributions $D_i, \dots , D_n$. Note that $\tau'_i=\mathcal{R}^{i+1}_{D'}(\pi^*(D')$.
    \begin{claim}
    If $\left |\tau'_i-\ex_{r_{i+1}\sim D_{i+1}}\left [\max\{\tau'_{i+1},r_{i+1}\}\right ]\right |\leq \epsilon$ for all $i$,  then 
    \[\mathcal{R}_{D'}(\pi^*(D'))\geq \mathcal{R}_{D}(\pi^*(D)-n\epsilon\]
    \end{claim}
    \begin{proof}
        We will use backward induction to show that for all $i$ we have 
        \[\mathcal{R}^{i}_{D'}(\pi^*(D'))\geq \mathcal{R}^{i}_{D}(\pi^*(D))-(n-i+1)\epsilon.\]
        For the base case of $i=n$ we have
        \begin{equation}\label{basecase}
        \begin{aligned}\mathcal{R}^{n}_{D'}(\pi^*(D')) 
        &=\ex_{r'_n\sim D'_n}[\max\{0,r'_n\}]\\
        &\geq \ex_{r_n\sim D_n}[\max\{0,r_n\}]-\epsilon =\mathcal{R}^{n}_{D}(\pi^*(D)\end{aligned}\end{equation} so the claim holds. 
        In the induction step we consider a fixed $i\in [n-1]$ and assume the statement holds for $i+1$. Now we have 
        \[\begin{aligned}\mathcal{R}^{i}_{D'}(\pi^*(D')) &= \ex_{r'_i\sim D'_i}\left [\max\{\tau'_i, r'_i\}\right ]\\
        &\geq \ex_{r_i \sim D_i}\left [\max\{\tau'_i, r_i\}\right ] -\epsilon\\
        &= \ex_{r_i \sim D_i}\left [\max\{\mathcal{R}^{i+1}_{D'}(\pi^*(D')), r_i\}\right ] -\epsilon\\
        &\geq \ex_{r_i \sim D_i}\left [\max\{\mathcal{R}^{i+1}_{D}(\pi^*(D))-(n-i)\epsilon, r_i\}\right ] -\epsilon\\
        &\geq \ex_{r_i\sim D_i}\left [\max\{\mathcal{R}^{i+1}_{D}(\pi^*(D)), r_i\}\right ] -(n-i+1)\epsilon\\
        &=\mathcal{R}^{i}_{D}(\pi^*(D))- (n-i+1)\epsilon,
        \end{aligned}\]
        which proves the Claim.
        \end{proof}
    In the second step we want to bound the difference in the reward when applying the same policy to slightly different distributions. 
   \begin{claim}
       If $\left |\tau'_i-\ex_{r_{i+1}\sim D_{i+1}}\left [\max\{\tau'_{i+1},r_{i+1}\}\right ]\right |\leq \epsilon$ for all $i$,  then 
    \[\mathcal{R}_{D}(\pi^*(D'))\geq \mathcal{R}_{D'}(\pi^*(D'))-n\epsilon\]
   \end{claim}
   \begin{proof}
       Again we will use backward induction to show that for all $i$ we have  \[\mathcal{R}^{i}_{D}(\pi^*(D'))\geq \mathcal{R}^{i}_{\bm{\lwd},x}(\pi^*(D'))-(n-i+1)\epsilon.\]
       The base case of $i=n$ follows immediately from inequality \eqref{basecase}. In the induction step we consider a fixed $i\in [n-1]$ and assume the statement holds for $i+1$. Now given a $r_i\sim D_i$ the expected reward of policy $\pi^*(D')$ is $r_i$ if this reward is above the threshold. Otherwise it is the expected reward obtained on distributions $D_{i+1}, \dots , D_n$. We can use this to receive the following result. 
        \[
			 \begin{aligned}
			 	\mathcal{R}^{i}_{D}(\pi^*(D'))=&\ex _{r\sim D}\left [r_i \mathbbm{1}\{r_i\geq \tau'_i(x)\} + \mathcal{R}^{i+1}_{D}(\pi^*(D')) \mathbbm{1}\{r_i < \tau'_i\}   \right ]\\
			 	\geq &\ex_{r\sim D'} \left [r_i \mathbbm{1}\{r_i\geq \tau'_i\} + \mathcal{R}^{i+1}_{D'}(\pi^*(D')) \mathbbm{1}\{r_i < \tau'_i\}   \right ] -(n-i)\epsilon\\
			 	\geq & \ex_{r\sim D} \left [\max\{\tau'_i, r_i\}    \right] -(n-i)\epsilon \\
			 	\geq & \ex _{r'\sim D'}\left [\max\{\tau'_i, r'_i \}\right ] - (n-i+1)\epsilon\\
			 	= & \mathcal{R}^{i}_{D'}(\pi^*(D')) - (n-i+1)\epsilon
			 \end{aligned}
			 \]
   \end{proof}
   Putting both claims together proves the Lemma. 
\end{proof}		
  Hence given an instance of contextual Optimal Stopping with unknown weight distributions and $m=O\left (\frac{d\xi^2c^4_{\max}}{\epsilon^8\delta^2}\right )$ samples from each distribution, we can learn weight distributions $\lwd_1, \dots \lwd_n$, such that with probability of at least $1-n\delta$ we we can approximate the optimal policy with an error of at most $2n\epsilon$.   

\bibliography{bibliography}
\appendix
\section{Proof of Lemma \ref{gen_distr}}\label{proof_gen_distr}
\begin{lemmano}
    For any distribution $\rwd$ over weight vectors $v\in [0,1]^d$ there exists a distribution $V$ with a support of at most $2\frac{c_{\max}^3}{\epsilon^2}\left (d\log\left (dc_{\max}\right ) + (d+1)\log \left (2/\epsilon\right )\right ) $ vectors, such that for any context $x\in [0,1]^d$ we have $L^x(V) \leq L^x(\rwd) + \epsilon$.
\end{lemmano}
\begin{proof}
    
We will first focus on showing, that there exists a distribution $V$, such that the capped expectation with respect to $V$ $\ex_{v\sim V}\left [\max\{c,f(v,x)\}\right ]$ is close to the capped expectation with respect to $\rwd$, namely $\ex_{v^*\sim \rwd}\left [\max\{c,f(v^*,x)\}\right ]$ for all $x$ and all $c$. To give an intuition for this, consider a set of samples $S=\{v_1, \dots, v_n\}$ drawn from $\rwd$. If for some $x$ and $c$ the empirical capped expectation on these samples $\frac{1}{n}\sum_{i=1}^n{\max\{c,f(v_i,x)\}}$ is close to the capped expectation on $\rwd$, then the same is true for the capped expectation on the uniform distribution on $S$. Hence we will concentrate on finding a number of samples $n$, such that
\[\left |\frac{1}{n}\sum_{i=1}^n{\max\{c,f(v_i,x)\}}-\ex_{v^*\sim \rwd}\left [\max\{c,f(v^*,x)\}\right ]\right|\] is small with positive probability. 
 
We will rely on Hoeffding's inequality \citet{Hoeffding} to bound the probabilty that an empirical value on a set of samples is far away from the expected value. 
        \begin{theorem}{\textbf{Hoeffding's inequality}}
            Let $X_1, \dots X_n$ be independent random variables such that $a_i\leq X_i\leq b_i$. Consider the sum of these random variables $S_n=\sum_{i=1}^n{X_i}$. Then Hoeffding's inequality states that, for all $t>0$, 
            \[\prob \left (\left\lvert\frac{1}{n}S_n-\frac{1}{n}\ex\left[S_n\right ]\right\rvert\geq t\right )\leq 2\exp{\left (-\frac{2t^2n^2}{\sum_{i=1}^n{(b_i-a_i)^2}}\right )}\]
    \end{theorem}
      
Let $n={\ensuremath{2\frac{c_{\max}^3}{\epsilon^2}\log\left (2\frac{(2dc_{\max})^d}{\epsilon ^{d+1}}\right )}} $, $\gamma=\frac{\epsilon}{2\sqrt{c_{\max}}}$ and $v_1, \dots v_n$ be samples from $\rwd$. Using Hoeffding's inequality we get that for any $c\in C_\epsilon$ and any context $x$ we have 
\[\prob\left (\left |\frac{1}{n}\sum_{i=1}^n{\max\{c,f(v_i,x)\}}-\ex _{v\sim \rwd}\left [\max\{c,f(v,x)\}\right ]\right |\geq \gamma\right )< \frac{\epsilon \gamma^d}{d^{d}c_{\max}}.\]
We define the discrete set $X_{\gamma, d}=\{i\frac{\gamma}{d}|i\in \{0,\dots,\frac{d}{\gamma}\}\}^d$ for the context vectors. We then have that $\frac{\epsilon \gamma^d}{d^{d}c_{\max}}=\frac{1}{|C_{\epsilon}||X_{\gamma, d}|}$ and using the union bound we get that 
\[\prob\left (\exists x\in X_{\gamma,d},c, \exists c\in C_\epsilon: \left |\frac{1}{n}\sum_{i=1}^n{\max\{c,f(v_i,x)\}}-\ex _{v\sim \mathcal{D}}\left [\max\{c,f(v,x)\}\right ]\right |\geq \gamma \right )< 1.\]
Hence the probability that for all $x\in X$ and all $c\in C$ the sampled vectors approximate the expected value with an error of less than $\gamma$ is greater than 0, which implies that there exists a set of samples $S=\{v_1, \dots v_n\}$ such that the uniform distribution on $S$ approximates $\rwd$ accordingly. Let $V$ denote that uniform distribution on $S$.

At last, for any $x \in [0,1]^d$ there exists some $x'\in X_{\gamma,d}$, such that $\left |\max\{c,f(v,x)\}-\max\{c,f(v,x')\}\right |\leq \gamma$ as $f$ is $\sqrt{d}$-Lipschitz in the context. Hence we get an additional error of at most $\gamma$ for contexts, which are not part of the discrete sets. 

Now we come back to the true loss. For some $x\in [0,1]^d$ we have 
\[
\begin{aligned}L^x(V)=&\ex_{v^*\sim \rwd}\left [\sum_{c\in C_\epsilon}{\left (\ex_{v\sim V}[\max\{c,f(v,x)\}]-\max\{c,f(v^*,x)\}\right )^2}\right ]\\
\leq & \ex_{v^*\sim \rwd}\left [\sum_{c\in C_\epsilon}{\left (\ex_{v'\sim \rwd}[\max\{c,f(v',x)\}]-\max\{c,f(v^*,x)\}\right )^2+(2\gamma)^2}\right ]\\
\leq &\ex_{v^*\sim \rwd}\left [\sum_{c\in C_\epsilon}{\left (\ex_{v'\sim \rwd}[\max\{c,f(v',x)\}]-\max\{c,f(v^*,x)\}\right )^2}\right ]+\frac{c_{\max}}{\epsilon}\left (\frac{\epsilon}{\sqrt{c_{\max}}}\right )^2\\
\leq &\ex_{v^*\sim \rwd}\left [\sum_{c\in C_\epsilon}{\left (\ex_{v'\sim \rwd}[\max\{c,f(v',x)\}]-\max\{c,f(v^*,x)\}\right )^2}\right ]+\epsilon,
\end{aligned}\]
which finishes the proof.
\end{proof}
\section{Convex Learning}\label{convex_learning}
 A learning problem, as defined by \citet{vapniklearning}, is specified by a hypothesis space $\mathcal{H}$, an instance set $\mathcal{Z}$ and a loss function $\ell:\mathcal{H}\times\mathcal{Z}\rightarrow \mathbb{R}$. In some cases the instance set is defined by a domain set $\contextdist$ and a label set $\mathcal{Y}$ such that $\mathcal{Z}=\contextdist\times \mathcal{Z}$. Given a distribution $\mathcal{D}$ on $\mathcal{Z}$ the quality of each hypothesis $h\in \mathcal{H}$ is measured by its true loss $L_\mathcal{D}(h)=\ex_{z\sim \mathcal{D}}\left [\ell(h,z)\right ]$. While $\mathcal{H}$, $\mathcal{Z}$ and $\ell$ are known, the distribution $\mathcal{D}$ is unknown to the learner. 
    
    The objective is to find a hypothesis $h^*\in \mathcal{H}$ which minimizes the true loss. Since the distribution is unknown, we need to rely on a finite set of samples $S=\{z_1, \dots , z_m\}$ with $S\sim \mathcal{D}^m$. 
    A convex learning problem is a learning problem with the added constraints that the loss function $\ell(h,z)$ is Lipschitz-continuous and convex in $h$ for every $z$, and that $\mathcal{H}$ is closed, convex and bounded. 
    The following theorem, shown by \citet{JMLR:shalev-shwartz}, captures the learnability of convex learning problems. 
    \begin{theorem}
        Let $(\mathcal{H},\mathcal{Z},\ell)$ be a convex learning problem such that $\ell(h,z)$ is $\rho$-Lipschitz with respect to $h$ and $\|h\|_2\leq B$ for all $h\in \mathcal{H}$. Let $S=\{z_1, \dots, z_m\}$ be a sample set and 
        \[h'(S)=\arg \min _{h\in \mathcal{H}} \left ( \frac{1}{m}\sum_{i=1}^m{\ell(h,z_i)} +\lambda \|h\|^2_2 \right )\]
        with $\lambda = \sqrt{\frac{2\rho ^2}{B^2 m}}$.
        It then follows that 
        \[\ex_S\left [L_\mathcal{D}(h'(S))\right ]\leq \min_{h\in \mathcal{H}}L_\mathcal{D}(h)+\rho B \sqrt{\frac{8}{m}}.\]
        In particular, for every $\epsilon >0$, if $m\geq \frac{8\rho ^2 B^2}{\epsilon ^2}$ then $\ex_S\left [ L_\mathcal{D}(h'(S))\right ]\leq \min_{h\in \mathcal{H}}L_\mathcal{D}(h)+ \epsilon$. 
    \end{theorem}
    \par
\textbf{Example: Linear Regression}
Probably the best known convex learning problem is linear regression. In linear regression we have a domain set $X\subseteq \mathbb{R}^d$ and a label set $\mathcal{Y}\subseteq\mathbb{R}$ which define our instance set $\mathcal{Z}_{\text{reg}}$. Here we will assume that the samples are generated in the same way as in contextual learning, namely there exists a distribution $\rwd$ on vectors in $\mathbb{R}^d$ and a distribution $\contextdist$ on $X$. We then draw a context $x$ from $\contextdist$ and a vector $v$ from $\rwd$ which results in the sample $(x,\langle v,x\rangle)$. The distribution over the samples is denoted by $\mathcal{D}$. The goal of linear regression is to learn a vector $w\in \mathbb{R}^d$ which, given a context $x\in \contextdist$, approximates the expected label $\ex_{v\sim \rwd}[\langle v,x\rangle]\approx \langle w,x\rangle$. The hypothesis space is therefore defined as $\mathcal{H}_\text{reg}\subseteq\mathbb{R}^d$. 
Finally we have $\ell_\text{reg}(w,(x,y))=(\langle w,x\rangle-y)^2$ which is called the squared loss.\\

\end{document}